\documentclass{article} 
\usepackage[usenames,dvipsnames]{color}
\usepackage{amsmath} 
\usepackage{amssymb} 
\usepackage{bm}
\usepackage{amsthm}
\usepackage{CJK}
\usepackage{indentfirst}
\usepackage{tikz}
\usepackage{amsmath,amsfonts,amsthm,mathrsfs}
\usepackage{hyperref}
\numberwithin{equation}{section}
\usepackage{geometry}

\numberwithin{equation}{section}
\newtheorem{theorem}{Theorem}
\newtheorem{lemma}{Lemma}

\newtheorem{corollary}{Corollary}
\newtheorem{definition}{Definition}
\newtheorem{remark}{Remark}

\usepackage{algorithm,algpseudocode}
\usepackage{amsthm}

\def\NN{\mathbb{N}}
\def\RR{\mathbb{R}}
\def\BB{\mathbb{B}}

\def\CC{\mathbb{C}}

\begin{document}
	
	\title{Expressivity and Approximation Properties of Deep \\Neural Networks with ${\rm ReLU}^k$ Activation}
	\author{Juncai He\footnotemark[1] \and Tong Mao\footnotemark[1] \and  Jinchao Xu\footnotemark[1] \footnotemark[2]}
	\date{}                                           
	
	\maketitle
	\renewcommand{\thefootnote}{\fnsymbol{footnote}} 
	\footnotetext[1]{Computer, Electrical and Mathematical Science and Engineering Division, King Abdullah University of Science and Technology, Thuwal 23955, Saudi Arabia.} 
	\footnotetext[2]{Department of Mathematics, The Pennsylvania State University, University Park, PA 16802, USA.}

	\begin{abstract}
		In this paper, we investigate the expressivity and approximation properties of deep neural networks employing the ReLU$^k$ activation function for $k \geq 2$. 
		Although deep ReLU networks can approximate polynomials effectively, deep ReLU$^k$ networks have the capability to represent higher-degree polynomials precisely.
		Our initial contribution is a comprehensive constructive proof for polynomial representation using deep ReLU$^k$ networks. This allows us to establish an upper bound on both the size and count of network parameters. Consequently, we are able to demonstrate a suboptimal approximation rate for functions from Sobolev spaces as well as for analytic functions. Additionally, through an exploration of the representation power of deep ReLU$^k$ networks for shallow networks, we reveal that deep ReLU$^k$ networks can approximate functions from a range of variation spaces, extending beyond those generated solely by the ReLU$^k$ activation function. This finding demonstrates the adaptability of deep ReLU$^k$ networks in approximating functions within various variation spaces.
	\end{abstract}

	\section{Introduction}
	Over the past decade, neural networks have achieved remarkable success in a variety of fields, including image recognition \cite{lecun1998gradient, krizhevsky2012imagenet, he2016deep, he2019mgnet}, speech recognition \cite{mikolov2011strategies, hinton2012deep, sainath2013improvements}, and natural language processing \cite{collobert2011natural, bordes2014question, jean2014using, sutskever2014sequence}, among others. These achievements have sparked increased interest in the theoretical understanding of neural networks.
	Theoretically, the generalization error of a neural network can be decomposed into the sum of the approximation error and the sampling error, as detailed in works such as \cite{cucker2007learning}. In practice, the reduction of approximation error depends on carefully designed neural network architectures, whereas the control of sample error requires sufficient and reliable data resources. 
	For instance, the approximation capabilities of neural networks with sigmoid activation functions were explored in various works \cite{cybenko1989approximation, hornik1989multilayer, barron1993universal, mhaskar1995degree}. Additionally, the universality of neural networks with non-polynomial activation functions was established in \cite{leshno1993multilayer,chui1992approximation,mhaskar1993approximation}.
	In recent years, due to their simple computational form and the potential to overcome the gradient vanishing problem, neural networks employing the Rectified Linear Unit (ReLU, defined as $\text{ReLU}(x) = \max\{0,x\}$) and ReLU$^k = (\text{ReLU})^k$ activation functions have gained widespread use. 
	This paper focuses on the expressivity and approximation properties of neural networks with ReLU$^k$ activation functions.

	Given the fact that any ReLU networks are continuous piecewise linear (CPwL) functions, \cite{arora2018understanding} initially demonstrated that any CPwL functions on $[0,1]^d$ can be represented by deep ReLU networks with $\lceil \log_2(d+1)\rceil$ hidden layers. Subsequently, \cite{he2020relu} established that shallow ReLU networks cannot represent any piecewise linear function on $[0,1]^d$ with $d\ge2$. In contrast, deep ReLU networks with more layers (but fewer parameters) can recover any linear finite element functions. Recently, \cite{he2023optimal} explored the optimal representation expressivity for continuous piecewise linear functions using deep ReLU networks on $[0,1]$. By integrating this expressivity with the Kolmogorov Superposition Theorem \cite{kolmogorov1957representation}, a significantly improved approximation rate was achieved. 
	Most recently, through a detailed examination of the basis functions and high-dimensional simplicial mesh characteristics of Lagrange finite element functions, which are a special type of piecewise polynomial functions, the authors in \cite{he2023deep} demonstrated that Lagrange finite element functions of any order in arbitrary dimensions can be represented by DNNs with ReLU or ReLU$^2$ activation functions.
	
	On the other hand, the approximation theory of shallow and deep neural networks with ReLU and ReLU$^k$ activation functions has been studied from various aspects. 
	For functions from the Barron space, it has been proved in \cite{klusowski2018approximation} that shallow ReLU networks uniformly with the rate $\mathcal{O}\left(N^{-\frac{1}{2}-\frac{1}{d}}\right)$ with $\mathcal{O}(N)$ parameters. This result has been generalized to $L^p$-approximation, spectral Barron spaces, and ReLU$^k$ activation functions in \cite{siegel2021optimal,siegel2022high,siegel2022sharp,ma2022uniform}. In particular, a optimal rate $\tilde{\mathcal{O}}(N^{-\frac{1}{2}-\frac{2k+1}{2d}})$ for functions from the so-called spectral Barron spaces, where $\tilde{\mathcal{O}}$ indicates up to a constant power of $\log N$. Furthermore, for $\alpha$-H\"older functions, \cite{mao2023rates} proved the approximation rate $\tilde{\mathcal{O}}\left(N^{-\frac{\alpha}{d}\frac{d+2}{d+4}}\right)$ by shallow ReLU networks. This result is generalized to shallow ReLU$^k$ networks for $k=0,1,2,\dots$ in \cite{yang2023optimal}. In particular, for $k=0,1$, the rate is improved to the optimal rate $\mathcal{O}\left(N^{-\frac{\alpha}{d}}\right)$.
	Compared to shallow neural networks, deep network architectures have significant advantages as demonstrated in \cite{chui1996limitations,delalleau2011shallow,safran2016depth,eldan2016power}. Specifically, by using a bit-extraction technique, it has been proved that fully-connected deep ReLU networks with depth $L$ and width $N$ attain the optimal approximation rate $\tilde{\mathcal{O}}\left(N^2L^2\right)$ \cite{shen2019deep,shen2022optimal,siegel2022optimal}. 
	
	For functions with specific compositional structures, deep ReLU networks can achieve significantly better expressivity power or approximation accuracy than their shallow ReLU counterparts \cite{poggio2017and,song2023approximation,mao2023approximating}. 
	In particular, it has been shown that any function constructed by a shallow ReLU neural network can be represented by a deep ReLU neural network with width $d+4$ \cite{lu2017expressive} or a deep convolutional neural network whenever the filter size $\geq2$ \cite{zhou2020universality,he2022approximation}.
	In addition, it was demonstrated in \cite{telgarsky2015representation} that sawtooth functions on $[0,1]$ with $2^N$ oscillations can be represented by a deep ReLU network with $\mathcal{O}(N)$ parameters. In contrast, a shallow ReLU network would require $\mathcal{O}(2^N)$ parameters to construct such a function. Subsequent studies in \cite{lu2017expressive,yarotsky2017error} showed that multivariate polynomials could be approximated by deep ReLU networks exponentially, achieving an approximation rate of $\tilde{\mathcal{O}}(N^{-\frac{\alpha}{d}})$ for functions from $C^\alpha(\BB^d)$. Later, a further study of approximating holomorphic functions by deep ReLU networks has been introduced in \cite{opschoor2022exponential}. Additionally, \cite{montanelli2019new} demonstrated that the sparse grid basis can be exponentially approximated by deep ReLU networks, leading to a suboptimal approximation rate of $\tilde{\mathcal{O}}(N^{-2})$ for functions from Korobov spaces. However, a foundational aspect of these results is a special expressivity property of deep ReLU networks. 
	A more in-depth and systematic analysis of this result, from the hierarchical basis perspective, is presented in \cite{he2022relu}.
	
	Given these findings, it is reasonable to anticipate that deep ReLU$^k$ networks might also offer superior representation and approximation properties compared to their shallow counterparts and deep ReLU networks. However, the expressivity and approximation properties of neural networks exclusively using the ReLU$^k$ activation function are not extensively documented in the literature.
	For example, it has been noted in studies such as \cite{chen2022power} and \cite{xu2020finite} that ReLU$^k$ networks can represent any global polynomial, with existence proofs provided to support this claim. Furthermore, the authors in \cite{li2019better} constructed a deep ReLU$^2$ network with $\mathcal{O}(n^d)$ neurons, demonstrating its capability to represent all polynomials of degree $\leq n$ on $\mathbb R^d$.
	In this paper, we generalize this to all ReLU$^k$ networks, providing a constructive proof for polynomial representation and estimating the upper bound of parameter size and count. Specifically, we show that shallow ReLU$^k$ networks with $\mathcal{O}(k^{d})$ parameters can represent polynomials of degree at most $k$ in $\mathbb R^d$, and that deep ReLU$^k$ networks of depth $L$ with $\mathcal{O}(k^{L d})$ parameters can represent polynomials of degree $\leq k^L$. Moreover, deep ReLU$^k$ networks can approximate analytic functions and functions from Sobolev spaces with suboptimal rates. Additionally, we prove that deep ReLU$^k$ networks with depth $L$ and width $2(k+1)n$ can represent functions constructed by any shallow ReLU$^{k^{\ell}}$ networks of width $n$ for any $\ell \le L$. Together with recent results on variation spaces \cite{siegel2022sharp}, we demonstrate an approximation error of $\mathcal{O}(N^{-\frac{1}{2}-\frac{2k^\ell+1}{2d}})$ for functions from the variation space $\mathcal{K}_1(\mathbb{P}_{k^\ell}^d)$ for any $\ell\leq L$. This underscores the adaptability of deep ReLU$^k$ networks in approximating functions from variation spaces without precise knowledge of their regularity, providing new insights into the advantages of deep architectures with ReLU$^k$ activations.
	
	The rest of this section will formally define our ReLU$^k$ network architectures, including bounds on parameter size and count. Section~\ref{sec:RepPoly} presents the constructive proof for global polynomial representation by ReLU$^k$ networks and discusses their approximation properties for analytic functions and functions from Sobolev spaces. Section~\ref{sec:AdaVaration} explores the adaptive approximation properties of deep ReLU$^k$ networks for functions in variation spaces. Section~\ref{sec:conclusions} offers concluding remarks and reflections on our findings.

	\begin{definition}
		Let $k\in\NN$, then we denote the univariate function {\rm ReLU}$^k$ as 
		\begin{equation}
			\sigma_k(x)=\left\{\begin{array}{ll}
				x^k, &\quad x\geq0,  \\
				0, &\quad x<0. 
			\end{array}\right.
		\end{equation}
		In particular, the {\rm ReLU} function is the {\rm ReLU}$^k$ function with $k=1$.
		
		For any vector $v=(v_1,\dots,v_n)$, we abuse the notation and denote the vector mapping
		\begin{equation}
			\sigma_k(v)=(\sigma_k(v_1),\dots,\sigma_k(v_n)) \in \mathbb R^n.
		\end{equation}
	\end{definition}

	\begin{definition}[ReLU$^k$ networks]\label{def:ReLUk}
		Let $d,k,n\in\NN^+$, $B>0$, $\BB^d$ be the unit ball of $\RR^d$ centered at the origin. Denote $\Sigma_{n}^k(B)$ be the class of shallow {\rm ReLU}$^k$ network on $\BB^d$
		\begin{equation}\label{eqn:defshallowReLUk}
			\Sigma_{n}^k(B)=\left\{x\mapsto c\cdot\sigma_k(Ax+b):\ A\in\RR^{n\times d},\ b\in\RR^n,\ c\in\RR^n,\ \|A\|_{\max}\leq B,\ \|b\|_\infty\leq B\right\}.
		\end{equation}
		where
		$$\|A\|_{\max}:=\max_{\substack{1\leq i\leq n\\1\leq j\leq d}}|(A)_{ij}|.$$
		In addition, if we suppose further in \eqref{eqn:defshallowReLUk} that $\|c\|_\infty<M$ for some $M>0$, we will denote the set as $\Sigma_{n,M}^k(B)$. This coincides with the notation in the previous works \cite{siegel2021optimal,siegel2022high,siegel2022sharp}.
		
		The class of fully-connected deep {\rm ReLU}$^k$ networks with depth $L$ and width $\{n_i\}_{i=1}^{L}$ is defined inductively as
		\begin{equation}\label{eqn:def_deepnet}
			\begin{split}
				\Sigma_{n_{1:L}}^k(B)=\{x\mapsto c\cdot h_{L}(x):&\ h_{i+1}(x)=\sigma_k\left(A_{i}h_i(x)+b_i\right),\ A_{i}\in\RR^{n_{i}\times n_{i-1}},\ b_i\in\RR^{n_{i}},\ c\in\RR^{n_{L}},\\
				&\ \|A_{i}\|_{\max}\leq B,\ \|b_i\|_\infty\leq B, i=1,\dots,L\},
			\end{split}
		\end{equation}
		where $n_0=d$, $h_0(x)=x$.
		
		Similarly, if we further require the weights in the output layers to be bounded as $\|c\|_\infty\leq M$, then we denote the class as $\Sigma_{n_{1:L},M}^k(B)$.
	\end{definition}

	Before introducing our results, we specify that the number of parameters in the fully connected deep ReLU$^k$ network is $n_L+\sum\limits_{i=1}^{L}n_{i}(n_{i-1}+1)$. However, deep neural network structures are often considered to be sparse (e.g.,\cite{chui2019deep,montanelli2019new,mao2021theory,fang2020theory}), i.e., most of the parameters are fixed to be $0$. Such a network is called a sparse deep network. The sparsity remarkably decreases the number of parameters and, thus, the computation cost. The deep ReLU$^k$ networks we constructed in this paper are sparse neural networks.


	\section{Representing Polynomials}\label{sec:RepPoly}
	In this section, we construct the ReLU$^k$ neural networks that represent polynomials. We will first show how shallow ReLU$^k$ networks represent polynomials of degree bounded by $k$, and extend the result to deep networks. As a consequence, we obtain suboptimal approximation rates for analytic functions and functions from Sobolev spaces.
	
	The following lemma shows how monomials with the form $x^\alpha$ of degree $n$ are constructed by $n$-th powers of linear combinations of $(x_1,\dots,x_d)$ (see, e.g., \cite{chui1992approximation,mhaskar1993approximation}). It has been known for many years that the homogeneous polynomial space of degree $n$ is a span of such linear combinations. However, our following lemma construct it explicitly, which allows an estimation of the coefficients.
	\begin{lemma}\label{lem:represent_poly}
		Let $n,d\in\NN^+$, and $\alpha=(\alpha_1,\dots,\alpha_d)\in\NN^d$ with
		$$\alpha_1+\dots+\alpha_d=n,$$
		then the monomial $x^\alpha$ can be written as a linear combination
		\begin{equation}\label{eqn:monomial_comb}
			x^\alpha=\sum\limits_{n_2,\dots,n_d=-\left\lfloor\frac{n}{2}\right\rfloor}^{n-\left\lfloor\frac{n}{2}\right\rfloor}c_{n_2,\dots,n_d}\left(x_1+n_2x_2+\dots+n_dx_d\right)^{n}
		\end{equation}
		with each $c_{n_2,\dots,n_d}\in\left[-\left(\frac{n}{2}+1\right)^{2d},\left(\frac{n}{2}+1\right)^{2d}\right]$.
	\end{lemma}
	
	\begin{proof}
		We prove this inductively on $d$. Clearly, this is true for $d=1$. We assume the result is true for $d-1$, thus
		\begin{equation}\label{eqn:indu_hypoth}
			x_1^{\alpha_1}\dots x_{d-1}^{\alpha_{d-1}}=\sum\limits_{n_2,\dots,n_{d-1}=-\left\lfloor\frac{j}{2}\right\rfloor}^{j-\left\lfloor\frac{j}{2}\right\rfloor}c_{n_2,\dots,n_{d-1}}\left(x_1+n_2x_2+\dots+n_{d-1}x_{d-1}\right)^{j},
		\end{equation}
		where $j=n-\alpha_d$ and $c_{n_2,\dots,n_{d-1}}\in\left[-\left(\frac{n}{2}+1\right)^{2(d-1)},\left(\frac{n}{2}+1\right)^{2(d-1)}\right]$.
		
		Let $\xi(x)$ be any function of $x$ and consider $(\xi(x)+n_dx_d)^n$ for $n_d=-\left\lfloor\frac{n}{2}\right\rfloor,\dots,n-\left\lfloor\frac{n}{2}\right\rfloor$, then we can write
		$$(\xi(x)+n_dx_d)^n=\sum\limits_{i=0}^n\binom{n}{i}n_d^{n-i}\xi(x)^ix_d^{n-i}.$$
		Consequently, for any $b=(b_0,\dots,b_{n})^T \in\mathbb{R}^{n+1}$, we have
		\begin{equation}\label{eqn:matix_represent_poly}
			\sum\limits_{s=0}^{n}b_s\left[\xi(x)+\left(s-\left\lfloor\frac{n}{2}\right\rfloor\right)x_d \right]^n=b^TB_n\lambda
		\end{equation}
		where $\lambda=\left(\xi(x)^n,x_d\xi(x)^{n-1},\dots,x_d^{n-1}\xi(x),x_d^n\right)^\top$, and
		\begin{equation}\label{eqn:A_explicit}
			\begin{split}
				B_n=&\left[\begin{array}{cccc}
					\binom{n}{0}1&\binom{n}{1}\left(-\left\lfloor\frac{n}{2}\right\rfloor\right)&\dots&\binom{n}{n}\left(-\left\lfloor\frac{n}{2}\right\rfloor\right)^n \\
					\binom{n}{0}1&\binom{n}{1}\left(1-\left\lfloor\frac{n}{2}\right\rfloor\right)&\dots&\binom{n}{n}\left(1-\left\lfloor\frac{n}{2}\right\rfloor\right)^n \\
					\vdots&\vdots&\ddots&\vdots \\
					\binom{n}{0}1&\binom{n}{1}\left(n-\left\lfloor\frac{n}{2}\right\rfloor\right)&\dots&\binom{n}{n}\left(n-\left\lfloor\frac{n}{2}\right\rfloor\right)^n
				\end{array}\right]\\
				=&\left[\begin{array}{cccc}
					1&\left(-\left\lfloor\frac{n}{2}\right\rfloor\right)&\dots&\left(-\left\lfloor\frac{n}{2}\right\rfloor\right)^n \\
					1&\left(1-\left\lfloor\frac{n}{2}\right\rfloor\right)&\dots&\left(1-\left\lfloor\frac{n}{2}\right\rfloor\right)^n \\
					\vdots&\vdots&\ddots&\vdots \\
					1&\left(n-\left\lfloor\frac{n}{2}\right\rfloor\right)&\dots&\left(n-\left\lfloor\frac{n}{2}\right\rfloor\right)^n
				\end{array}\right]\left[\begin{array}{cccc}
					\binom{n}{0}&0&\dots&0\\
					0&\binom{n}{1}&\dots&0\\
					\vdots&\vdots&\ddots&\vdots \\
					0&0&\dots&\binom{n}{n}
				\end{array}\right].
			\end{split}
		\end{equation}
		Clearly, $B_n$ is invertible. Let $\hat b=e_{n-j+1}^TB_n^{-1}$ with
		$$e_{n-j+1}=\left(\underbrace{0,\dots,0}_{n-j},1,\underbrace{0,\dots,0}_{j}\right)^T$$
		and $\xi(x)=x_1+n_2x_2+\dots+n_{d-1}x_{d-1}$, then we have
		\begin{equation*}
			\begin{split}
				&\sum\limits_{s=0}^{n}\hat b_s^T\left[x_1+n_2x_2+\dots+n_{d-1}x_{d-1}+\left(s-\left\lfloor\frac{n}{2}\right\rfloor\right)x_d\right]^n=\hat b^T B_n\lambda=e_{n-j+1}\lambda=\xi(x)^jx_d^{n-j}\\
				=&(x_1+n_2x_2+\dots+n_{d-1}x_{d-1})^jx_d^{\alpha_d}.
			\end{split}
		\end{equation*}
		
		Substituting \eqref{eqn:indu_hypoth}, we get
		\begin{equation*}
			\begin{split}
				&\sum\limits_{n_2,\dots,n_{d-1}=-\left\lfloor\frac{j}{2}\right\rfloor}^{j-\left\lfloor\frac{j}{2}\right\rfloor}c_{n_2,\dots,n_{d-1}}\sum\limits_{s=0}^n\hat b_s\left[x_1+n_2x_2+\dots+n_{d-1}x_{d-1}+\left(s-\left\lfloor\frac{n}{2}\right\rfloor\right)x_d\right]^n\\
				=&\sum\limits_{n_2,\dots,n_{d-1}=-\left\lfloor\frac{j}{2}\right\rfloor}^{j-\left\lfloor\frac{j}{2}\right\rfloor}c_{n_2,\dots,n_{d-1}}(x_1+n_2x_2+\dots+n_{d-1}x_{d-1})^jx_d^{\alpha_d}\\
				=&x_1^{\alpha_1}\dots x_{d-1}^{\alpha_{d-1}}x_d^{\alpha_d}.
			\end{split}
		\end{equation*}
		For $n_2,\dots,n_d\in\left\{-\left\lfloor\frac{n}{2}\right\rfloor,\dots,n-\left\lfloor\frac{n}{2}\right\rfloor\right\}$, denote
		$$c_{n_2,\dots,n_{d}}:=\left\{\begin{array}{ll}
			c_{n_2,\dots,n_{d-1}}\hat b_{n_d+\left\lfloor\frac{n}{2}\right\rfloor}, &\quad  n_2,\dots,n_{d-1}\in\left\{-\left\lfloor\frac{j}{2}\right\rfloor,\dots,j-\left\lfloor\frac{j}{2}\right\rfloor\right\}\\
			0, & \quad\hbox{otherwise,}
		\end{array}\right.$$
		then
		$$\sum\limits_{n_2,\dots,n_d=-\left\lfloor\frac{n}{2}\right\rfloor}^{n-\left\lfloor\frac{n}{2}\right\rfloor}c_{n_2,\dots,n_d}\left(x_1+n_2x_2+\dots+n_dx_d\right)^{n}=x_1^{\alpha_1}\dots x_d^{\alpha_d}.$$
		
		Finally, let us consider the bounds of $c_{n_2,\dots,n_d}$. Here, we only need to estimate the bounds of $\hat b_{s}$. Since $\hat b=e_{n-j+1}^TB_n^{-1}$, for $s=0,\dots,n$, we have $|b_s|\leq\|B_n^{-1}\|_{\max}=\max\limits_{1\leq i,k\leq n+1}|(B_n^{-1})_{ik}|$.
		
		We can bound $\|B_n^{-1}\|_{\max}$ by the norm of the inverse of the Vandermonde matrix in \eqref{eqn:A_explicit}. Then by \cite[Theorem 1]{gautschi1978inverses}, we have
		\begin{equation}\label{eqn:Bn_inverse_bound}
			\begin{split}
				\|B_n^{-1}\|_{\max}\leq&\frac{\prod\limits_{s=0}^n\left(1+\left|s-\left\lfloor\frac{n}{2}\right\rfloor\right|\right)}{\min\limits_{0\leq i\leq n}\left\{(1+\left|i-\left\lfloor\frac{n}{2}\right\rfloor\right|)\prod\limits_{s\neq i}|s-i|\right\}}=\frac{\prod\limits_{s=0}^n\left(1+\left|s-\left\lfloor\frac{n}{2}\right\rfloor\right|\right)}{(1+\left|0\right|)\prod\limits_{s\neq \left\lfloor\frac{n}{2}\right\rfloor}\left|s-\left\lfloor\frac{n}{2}\right\rfloor\right|}\\
				=&\left(\prod\limits_{m=1}^{n-\left\lfloor\frac{n}{2}\right\rfloor}\frac{1+m}{m}\right)\times1\times\left(\prod\limits_{m=1}^{\left\lfloor\frac{n}{2}\right\rfloor}\frac{1+m}{m}\right)\leq\left(\frac{n}{2}+1\right)^2.
			\end{split}
		\end{equation}
		Thus
		$$|c_{n_2,\dots,n_d}|\leq|b_s||c_{n_2,\dots,n_{d-1}}|\leq\left(\frac{n}{2}+1\right)^2\left(\frac{n}{2}+1\right)^{2(d-1)}\leq\left(\frac{n}{2}+1\right)^{2d}.$$
		This completes the proof of Lemma \ref{lem:represent_poly}.
	\end{proof}
	
	By noticing that
	$$t^k=\sigma_k(t)+(-1)^k\sigma_k(-t),\quad t\in\RR,$$
	Lemma \ref{lem:represent_poly} enables the representation of the polynomials of degree $\leq k$ by shallow ReLU$^k$ neural networks, which is also the key to showing the expressively of deep ReLU$^k$ neural networks.
	
	\begin{remark}
		In fact, the choice of $n_2,\dots,n_d$ in the proof of Lemma \eqref{lem:represent_poly} is not the only choice. From \eqref{eqn:A_explicit}, we can observe that our proof relies on writing $B_n$ as a product Vandermonde matrix and a diagonal matrix. The attempt to lessen the bound \eqref{eqn:Bn_inverse_bound} resulted in our choice of $n_2,\dots,n_d$. However, we specify that our choice does not necessarily lead to the optimal upper bound $M$ in the next lemma.
	\end{remark}

	\begin{lemma}\label{lem:shallow_poly}
		Let $d,k\in\NN^+$, $k\geq2$, $N=2(k+1)^d$, and $B>0$, then any polynomial
		\begin{equation*}
			P(x)=\sum\limits_{\{\alpha\in\NN^d:\ \|\alpha\|_1\leq k\}}a_\alpha x^\alpha,\ x\in\BB^d
		\end{equation*}
		with degree $\leq k$ can be represented by the shallow {\rm ReLU}$^k$ network $\Sigma_{N,M}^k(B)$, where
		$$M=B^{-k}\left(\frac{k}{2}+1\right)^{2(d+1)+k}\left(\sum\limits_{\{\alpha\in\NN^d:\ \|\alpha\|_1\leq k\}}|a_\alpha|\right).$$
	\end{lemma}

	\begin{proof}
		By applying Lemma \ref{lem:represent_poly} with $d\leftarrow d+1$ and $n\leftarrow k$, we conclude
		any monomial of degree $k$ can be represented as
		$$x_1^{\alpha_1}\dots x_{d}^{\alpha_d}x_{d+1}^{\alpha_{d+1}}=\sum\limits_{n_2,\dots,n_d=-\left\lfloor\frac{k}{2}\right\rfloor}^{k-\left\lfloor\frac{k}{2}\right\rfloor}c_{n_2,\dots,n_{d+1}}\left(x_1+n_2x_2+\dots+n_dx_d+n_{d+1}x_{d+1}\right)^{k}.$$
		Replacing $x_{d+1}$ by $1$, then all monomials of degree $\leq k$ on variables $(x_1,\dots,x_d)$ can be represented as
		\begin{equation}\label{eqn:comb_monomial}
			x_1^{\alpha_1}\dots x_{d}^{\alpha_d}=\sum\limits_{n_2,\dots,n_d=-\left\lfloor\frac{k}{2}\right\rfloor}^{k-\left\lfloor\frac{k}{2}\right\rfloor}c_{n_2,\dots,n_d}\left(x_1+n_2x_2+\dots+n_dx_d+n_{d+1}\right)^{k}
		\end{equation}
		with $c_{n_2,\dots,n_{d+1}}=c_{n_2,\dots,n_{d+1}}(\alpha)\in\left[-\left(\frac{k}{2}+1\right)^{2(d+1)},\left(\frac{k}{2}+1\right)^{2(d+1)}\right]$.
		Then any polynomial
		$$P(x)=\sum\limits_{\{\alpha\in\NN^d:\ \|\alpha\|_1\leq k\}}a_\alpha x^\alpha$$
		can be written as
		\begin{equation*}
			\begin{split}
				P(x)=&\sum\limits_{\{\alpha\in\NN^d:\ \|\alpha\|_1\leq k\}}a_\alpha\sum\limits_{n_2,\dots,n_{d+1}=-\left\lfloor\frac{k}{2}\right\rfloor}^{k-\left\lfloor\frac{k}{2}\right\rfloor}c_{n_2,\dots,n_{d+1}}(\alpha)\left(x_1+n_2x_2+\dots+n_dx_d+n_{d+1}\right)^{k}\\
				=&\sum\limits_{n_2,\dots,n_d=-\left\lfloor\frac{k}{2}\right\rfloor}^{k-\left\lfloor\frac{k}{2}\right\rfloor}\beta_{n_2,\dots,n_{d+1}}\left(x_1+n_2x_2+\dots+n_dx_d+n_{d+1}\right)^{k}
			\end{split}
		\end{equation*}
		where
		\begin{equation*}
			\begin{split}
				\beta_{n_2,\dots,n_{d+1}}=&\sum\limits_{\{\alpha\in\NN^d:\ \|\alpha\|_1\leq k\}}a_\alpha c_{n_2,\dots,n_{d+1}}(\alpha)\leq\sum\limits_{\{\alpha\in\NN^d:\ \|\alpha\|_1\leq k\}}|a_\alpha|\max\limits_{\{\alpha\in\NN^d:\ \|\alpha\|_1\leq k\}}|c_{n_2,\dots,n_{d+1}}(\alpha)|\\
				\leq&\left(\frac{k}{2}+1\right)^{2(d+1)}\left(\sum\limits_{\{\alpha\in\NN^d:\ \|\alpha\|_1\leq k\}}|a_\alpha|\right).
			\end{split}
		\end{equation*}
		
		Thus, the polynomial $P$ can be written as
		\begin{equation*}
			\begin{split}
				P(x)=\sum\limits_{n_2,\dots,n_d=-\left\lfloor\frac{k}{2}\right\rfloor}^{k-\left\lfloor\frac{k}{2}\right\rfloor}\beta_{n_2,\dots,n_{d+1}}&\left[\sigma_k\left(x_1+n_2x_2+\dots+n_dx_d+n_{d+1}\right)\right.\\
				&\ +\left.(-1)^k\sigma_k\left(-x_1-n_2x_2-\dots-n_dx_d-n_{d+1}\right)\right].
			\end{split}
		\end{equation*}
		By noticing the inequality $\sigma_k(ay)=a^k\sigma_k(y)$ for all $a>0$, $y\in\RR$, we can write $P$ as
		\begin{equation*}
			\begin{split}
				P(x)=\sum\limits_{n_2,\dots,n_d=-\left\lfloor\frac{k}{2}\right\rfloor}^{k-\left\lfloor\frac{k}{2}\right\rfloor}&B^{-k}\left\lceil\frac{k}{2}\right\rceil^{k}\beta_{n_2,\dots,n_{d+1}}\left[\sigma_k\left(B\left\lceil\frac{k}{2}\right\rceil^{-1}(x_1+n_2x_2+\dots+n_dx_d+n_{d+1})\right)\right.\\
				&\ +\left.(-1)^k\sigma_k\left(-B\left\lceil\frac{k}{2}\right\rceil^{-1}(x_1+n_2x_2+\dots+n_dx_d+n_{d+1})\right)\right].
			\end{split}
		\end{equation*}
		Clearly, the parameters satisfy
		\begin{equation*}
			\left|B\left\lceil\frac{k}{2}\right\rceil^{-1}\right|\leq B \quad \text{and}  \quad 
			\left|B\left\lceil\frac{k}{2}\right\rceil^{-1}n_j\right|\leq B
		\end{equation*}
		for all $j=2,\dots,d+1$, and
		$$
		\left|B^{-k}\left\lceil\frac{k}{2}\right\rceil^{k}\beta_{n_2,\dots,n_{d+1}}\right|\leq B^{-k}\left(\frac{k}{2}+1\right)^{2(d+1)+k}\left(\sum\limits_{\{\alpha\in\NN^d:\ \|\alpha\|_1\leq k\}}|a_\alpha|\right)
		$$
		for all $ -\left\lfloor\frac{k}{2}\right\rfloor\leq n_2,\dots,n_d\leq k-\left\lfloor\frac{k}{2}\right\rfloor$. This proves Lemma \ref{lem:shallow_poly}.
	\end{proof}

	Now, we are ready to prove our first main result, which reveals the way to construct polynomials by deep ReLU$^k$ neural networks, as well as an estimation of the upper bound of the parameters.
	
	\begin{theorem}\label{thm:poly_deep}
		Let $d,k,L\in\NN^+$, $B>0$, and $n_1=\dots=n_ L=2(k^ L+1)^d$, then any polynomial
		\begin{equation*}
			P(x)=\sum\limits_{\{\alpha\in\NN^d:\ \|\alpha\|_1\leq k^L\}}a_\alpha x^\alpha,\ x\in\BB^d
		\end{equation*}
		with degree $\leq k^L$ can be represented by a deep {\rm ReLU}$^k$ network architecture $\Sigma_{n_{1:L},M}^k(B)$ with
		$$M=B^{-k\frac{k^{L}-1}{k-1}}\left(\frac{k^L}{2}+1\right)^{2(d+1)+k^L}\left(\sum\limits_{\{\alpha\in\NN^d:\ \|\alpha\|_1\leq k^L\}}|a_\alpha|\right).$$
		Furthermore, there are $2(2L+d)(k^L+1)^d$ nonzero parameters in the neural network architecture.
	\end{theorem}

	\begin{proof}
		By Lemma \ref{lem:shallow_poly}, $P$ can be represented as
		\begin{equation}
			P(x)=\sum\limits_{j=1}^{2(k^L+1)^d}c_j\sigma_{k^L}(w_j\cdot x+b_j),
		\end{equation}
		where $w_j\in[-1,1]^d$, $b_j\in[-1,1]$ and $c_j\in[-M',M']$ for $j=1,\dots,2(k^L+1)^d$, where
		$$M'=\left(\frac{k^L}{2}+1\right)^{2(d+1)+k^L}\left(\sum\limits_{\{\alpha\in\NN^d:\ \|\alpha\|_1\leq k^L\}}|a_\alpha|\right).$$
		
		Given $h_0(x)=x$, there exists a matrix $A_1$ with $2d(k^L+1)^d$ parameters and bias term $b=(b_1,\dots,b_{2d(k^L+1)})$ such that
		\begin{equation*}
			h_{1,j}(x)=\sigma_{k}(Bw_j\cdot x+Bb_j)=B^k\sigma_{k}(w_j\cdot x+b_j),
		\end{equation*}
		$j=1,\dots,2(k^L+1)^d$.
		
		Given
		$$h_i(x)=B^{k\frac{k^{i}-1}{k-1}}\left(\sigma_{k^i}(w_1\cdot x+b_1),\dots,\sigma_{k^i}\left(w_{2(k^L+1)^d}\cdot x+b_{2(k^L+1)^d}\right)\right)^\top,$$
		the identity matrix $A_i=BI_{2(k^L+1)^d\times2(k^L+1)^d}$ with $2(k^L+1)^d$ parameters and bias term $b=(0,\dots,0)$ give
		\begin{equation*}
			\begin{split}
				h_{i+1}(x)=&\sigma_k(A_ih_i(x))=\sigma_k(B\times B^{k\frac{k^{i}-1}{k-1}}h_i(x)),\\
				=&\left(\sigma_k\left(B^{\frac{k^{i+1}-1}{k-1}}\sigma_{k^i}(w_1\cdot x+b_1)\right),\dots,\sigma_k\left(B^{\frac{k^{i+1}-1}{k-1}}\sigma_{k^i}\left(w_{2(k^L+1)^d}\cdot x+b_{2(k^L+1)^d}\right)\right)\right)^\top,\\
				=&B^{k\frac{k^{i+1}-1}{k-1}}\left(\sigma_{k^{i+1}}(w_1\cdot x+b_1),\dots,\sigma_{k^{i+1}}\left(w_{2(k^L+1)^d}\cdot x+b_{2(k^L+1)^d}\right)\right)^\top.
			\end{split}
		\end{equation*}
		Inductively, we have
		$h_L(x)=B^{k\frac{k^{L}-1}{k-1}}\left(\sigma_{k^{L}}(w_1\cdot x+b_1),\dots,\sigma_{k^{L}}\left(w_{2(k^L+1)^d}\cdot x+b_{2(k^L+1)^d}\right)\right)^\top$.
		Thus, with $\tilde c=B^{-k\frac{k^{L}-1}{k-1}}(c_1,\dots,c_{2(k^L+1)^d})$, we have
		$$\tilde c\cdot h_L(x)=\sum\limits_{j=1}^{2(k^L+1)^d}c_j\sigma_{k^L}(w_j\cdot x+u_j)=P(x).$$
		Here, we have
		$$\|\tilde c\|_\infty\leq B^{-k\frac{k^{L}-1}{k-1}}M'=M.$$
		This implies $P\in\Sigma_{n_{1:L}}^k(M)$ with $n_1=\dots=n_L=2(k^L+1)^d$.
		
		There are $2d(k^L+1)^d$ parameters in $A_1$ and $2(k^L+1)^d$ parameters in $A_2,\dots,A_L$. The number of parameters in bias terms is $L\times2(k^L+1)^d$, and $2(k^L+1)^d$ parameters in $\tilde c$. 
		
	\end{proof}

	\begin{remark}
		As we mentioned before, the existence of such representation of the polynomials has been mentioned in previous literature such as \cite{xu2020finite,chen2022power}. Those results rely on a generalized Vandermonde matrix theory for high-dimensional cases, while our proof in Lemma \ref{lem:represent_poly} reduces the problem into $1$-dimension Vandermonde matrices. This improvement allows an estimation of the coefficients in the linear combination \eqref{eqn:monomial_comb}, and hence the parameters in the neural networks.
	\end{remark}
	
	As a consequence of the classical polynomial approximation theory (e.g. \cite[Chapter 7]{devore1993constructive}), we can deduce the rate of approximation for analytic functions and for functions from Sobolev spaces.
	\begin{definition}[Sobolev spaces]
		Let $d\in\NN^+$ and $1\leq p\leq\infty$. For $r\in\NN$, the Sobolev space $W^{r,p}(\BB^d)$ is defined as
		\begin{equation}
			W^{r,p}(\BB^d)=\left\{f\in L^p(\BB^d):\ \max\limits_{\{\alpha\in\NN^d,\|\alpha\|_1=r\}}\|D^\alpha f\|_{L^p(\BB^d)}<\infty\right\}
		\end{equation}
		with the norm given by
		\begin{equation*}
			\|f\|_{W^{r,p}(\BB^d)}^p=\|f\|_{L^p(\BB^d)}^p+\sum\limits_{\{\alpha\in\NN^d,\|\alpha\|_1=r\}}\|D^\alpha f\|_{L^p(\BB^d)}^p,\quad f\in W^{r,p}(\BB^d).
		\end{equation*}
		For $r\notin\NN$, let $\theta=r-\lfloor r\rfloor$. The Sobolev semi-norm is
		\begin{equation*}
			|f|_{W^{r,p}}(\BB^d)=\left\{\begin{array}{ll}
				\max\limits_{\{\alpha\in\NN^d,\|\alpha\|_1=r\}}\left(\displaystyle\int_{\BB^d\times\BB^d} \frac{|D^\alpha f(x)-D^\alpha f(y)|^p}{|x-y|^{d+\theta p}}dxdy\right)^{1/p},  &\quad 1\leq p<\infty,  \\
				\max\limits_{\{\alpha\in\NN^d,\|\alpha\|_1=r\}}\sup\limits_{x,y\in\BB^d,x\neq y}\displaystyle \frac{|D^\alpha f(x)-D^\alpha f(y)|}{|x-y|^\theta},  & \quad p=\infty.
			\end{array}\right.
		\end{equation*}
		The space $W^{r,p}(\BB^d)$ is defined as
		\begin{equation}
			W^{r,p}(\BB^d)=\left\{f\in W^{\lfloor r\rfloor,p}(\BB^d):\ |f|_{W^{r,p}}(\BB^d)<\infty\right\}
		\end{equation}
		with the norm
		$$\|f\|_{W^{r,p}(\BB^d)}=\|f\|_{W^{\lfloor r\rfloor,p}(\BB^d)}+|f|_{W^{r,p}}(\BB^d).$$
	\end{definition}

	\begin{definition}[Analytic functions]
		Let $d\in\NN^+$, $\rho\in(0,1)$, $f$ is said to be an analytic function on $U_\rho:=\left\{z\in\CC^d:\ \left|z_j+\sqrt{z_j-1}\right|<\rho^{-1},\ j=1,\dots,d\right\}$ if it is complex differentiable at each $z\in U_\rho$.
	\end{definition}

	\begin{corollary}
		Let $d,k,L\in\NN^+$, $B>0$, and
		$$n_1=\dots=n_L=2(k^L+1)^d.$$
		There exists a sparse deep ReLU$^k$ network $\Sigma_{n_{1:L}}^k(B)$ with $2(2L+d)(k^L+1)^d$ parameters such that
		\begin{enumerate}
			\item[(a)] if $f\in W^{r,p}(\BB^d)$ for some $r>0$, then
			\begin{equation}
				\inf\limits_{g\in\Sigma_{n_{1:L}}^k(B)}\|f-g\|_{L^p(\BB^d)}\leq\inf\limits_{g\in\mathcal{P}_d^{k^L}}\|f-g\|_{L^p(\BB^d)}\sim \|f\|_{W^{r,p}(\BB^d)}k^{-L r};
			\end{equation}
			\item[(b)] if $f$ is analytic on $U_\rho:=\left\{z\in\CC^d:\ \left|z_j+\sqrt{z_j-1}\right|<\rho^{-1},\ j=1,\dots,d\right\}$, then for any $\epsilon>0$ with $\rho+\epsilon<1$,
			\begin{equation}
				\inf\limits_{g\in\Sigma_{n_{1:L}}^k(B)}\|f-g\|_{L^p(\BB^d)}\leq\inf\limits_{g\in\mathcal{P}_d^{k^L}}\|f-g\|_{L^p(\BB^d)}\lesssim \left(\max\limits_{w\in\partial \Gamma_{\rho+\epsilon}}\left|f\left(\frac{w+w^{-1}}{2}\right)\right|\right)k^{Ld}(\rho+\epsilon)^{k^L},
			\end{equation}
			where $\Gamma_{\rho+\epsilon}:=\left\{w\in\CC^d:\ |w_j|\leq(\rho+\epsilon)^{-1},\ j=1,\dots,d\right\}$. The corresponding constant depends only on $d,\rho$, and $\epsilon$.
		\end{enumerate}
	\end{corollary}
	\begin{proof}
		\begin{enumerate}
			\item [(a)] The first inequality is exactly Theorem \ref{thm:poly_deep}. The relation
			$$\inf\limits_{g\in\mathcal{P}_d^{k^L}}\|f-g\|_{L^p(\BB^d)}\sim \|f\|_{W^{r,p}(\BB^d)}k^{-L r}$$
			is the classical polynomial approximation theory (see, e.g., \cite[Chapter 7]{devore1993constructive}).
			\item [(b)] Again, the first inequality is Theorem \ref{thm:poly_deep}. For the second inequality, we apply the classical estimation for analytic functions (see, e.g., \cite[Chapter 7]{devore1993constructive} and \cite[Lemma A.1]{doctor2023encoding} for the multivariate version) and conclude
			\begin{equation*}
				\begin{split}
					\inf\limits_{g\in\mathcal{P}_d^{k^L}}\|f-g\|_{L^2(\BB^d)}^2\leq&\left(\max\limits_{w\in\partial U_{\rho+\epsilon}}|f(w+w^{-1})|\right)^2\sum\limits_{n=k^L+1}^\infty\binom{n+d-1}{d-1}(\rho+\epsilon)^{2n}\\
					\leq& C^2\left(\max\limits_{w\in\partial U_{\rho+\epsilon}}|f(w+w^{-1})|\right)^2k^{2L}(\rho+\epsilon)^{2k^L},
				\end{split}
			\end{equation*}
			where $C=C(d,\rho,\epsilon)$ is independent of $f$ and $L$.
		\end{enumerate}
	\end{proof}
	
	\section{Adaptive Approximation for Functions from Variation Spaces}\label{sec:AdaVaration}
	In this section, we show the adaptive approximation properties of deep ${
		\rm ReLU}^k$ neural networks for functions in variation spaces. This result is obtained by studying the representation power of deep ${\rm ReLU}^k$ neural networks for shallow networks.
	\begin{lemma}\label{lem:variation_space}
		Let $d,k,L,n\in\NN^+$, $M>0$, and $n_1=\dots=n_{L}=2(k+1)n$. Then there exists a deep {\rm ReLU}$^k$ neural network architecture $\Sigma_{n_{1:L},M'}^k(B')$ with
		$$N=[(4L-2)(k+1)+d]n$$
		nonzero parameters, such that all functions $f\in \Sigma^{K}_{n,M}(B)$ with $K\in\{k^1,\dots,k^L\}$ can be represented this neural network. Furthermore, the parameters in $\Sigma_{n_{1:L},M'}^k(B')$ are bounded as
		$$M'=\left(\frac{k}{2}+1\right)^4M,\quad B'=B\vee\left(\frac{k}{2}+1\right)^4.$$
	\end{lemma}
	
	\begin{proof}
		Let $1\leq\ell\leq L$, $K=k^{\ell}$, and $f\in\Sigma^{K}_{n,M}(B)$. Then
		$$f_n(x)=\sum\limits_{m=1}^nc_m\sigma_{
			K}(w_m\cdot x+u_m),\quad x\in\BB^d,$$
		where $|c_m|\leq M$, $|u_m|\leq B$, and $\|w_m\|_\infty\leq B$ holds for $m=1,\dots,n$.
		
		Given $h_0(x)=x$, there exists a sparse matrix $A_0\in\RR^{n_1\times d}$ with $(A_0)_{i,j}=0$ for $i>n$ and $b_0=(u_1,\dots,u_n)^\top$ such that
		\begin{equation}
			h_{1,j}(x)=\left\{\begin{array}{ll}
				\sigma_k\left(w_j\cdot x+u_j\right),&\quad j\leq n, \\
				0, & \quad \hbox{otherwise.}
			\end{array}\right.
		\end{equation}
		Clearly, all the parameters in $A_0$ and $b_0$ are bounded by $B'$.
		
		Given
		\begin{equation*}
			h_{i,j}(x)=\left\{\begin{array}{ll}
				\sigma_{k^i}(w_j\cdot x+u_j),&\quad j\leq n, \\
				0, & \quad \hbox{otherwise,}
			\end{array}\right.
		\end{equation*}
		let $A_{i}=B'I_{2(k+1)n\times2(k+1)n}$, and $b_i=(0,\dots,0)^\top$,
		\begin{equation*}
			h_{i+1,j}(x)=\left\{\begin{array}{ll}
				\sigma_{k^{i+1}}(w_j\cdot x+u_j),&\quad j\leq n, \\
				0, & \quad \hbox{otherwise.}
			\end{array}\right.
		\end{equation*}
		So we get
		\begin{equation*}
			h_{\ell,j}(x)=\left\{\begin{array}{ll}
				\sigma_{K}(w_j\cdot x+u_j),&\quad j\leq n, \\
				0, & \quad \hbox{otherwise.}
			\end{array}\right.
		\end{equation*}
		By \eqref{eqn:comb_monomial}, the function $y\mapsto y$ can be represented as a linear combination
		\begin{equation}\label{eqn:represent_id}
			y=\sum\limits_{t=0}^ka_t\left(y+t-\left\lfloor\frac{k}{2}\right\rfloor\right)^k=\sum\limits_{t=0}^ka_t\left[\sigma_k\left(y+t-\left\lfloor\frac{k}{2}\right\rfloor\right)+(-1)^k\sigma_k\left(-y-t+\left\lfloor\frac{k}{2}\right\rfloor\right)\right]
		\end{equation}
		with $|a_t|\leq \left(\frac{k}{2}+1\right)^4$.
		
		At the $\ell$-th layer, there exists a sparse matrix $A_{\ell}\in\RR^{n_{\ell+1}\times n_{\ell}}$ with
		\begin{equation*}
			\begin{split}
				&(A_{\ell})_{2(k+1)(m-1)+2t+1,m}=1,\\
				&(A_{\ell})_{2(k+1)(m-1)+2t+2,m}=-1,
			\end{split}
		\end{equation*}
		for $m=1,\dots,n,\ t=0,\dots,k$ and proper bias term $b_{\ell}$ with $\|b_{\ell}\|_\infty\leq\frac{k}{2}+1$ such that
		\begin{equation*}
			\begin{split}
				&h_{\ell+1,2(k+1)(m-1)+2t+1}(x)=\sigma_k\left(\sigma_{L}(w_m\cdot x+u_m)+t-\left\lfloor\frac{k}{2}\right\rfloor\right)\\
				&h_{\ell+1,2(k+1)(m-1)+2t+2}(x)=\sigma_k\left(-\sigma_{L}(w_m\cdot x+u_m)-t+\left\lfloor\frac{k}{2}\right\rfloor\right).
			\end{split}
		\end{equation*}
		for $m=1,\dots,n,\ t=0,\dots,k$.
		
		At the $(\ell+1)$-th layer, there exists a sparse matrix $A_{\ell+1}\in\RR^{n_{\ell+2}\times n_{\ell+1}}$ with 
		\begin{equation*}
			\begin{split}
				&(A_{\ell+1})_{m,2(k+1)(m-1)+2t+1}=a_t,\\
				&(A_{\ell+1})_{m,2(k+1)(m-1)+2t+2}=(-1)^ka_t
			\end{split}
		\end{equation*}
		and $b_{\ell+1}=(0,\dots,0)^\top$ such that
		\begin{equation*}
			\begin{split}
				h_{\ell+2,j}(x)=&\sum\limits_{t=0}^ka_t\left[\sigma_k\left(\sigma_{L}(w_j\cdot x+u_j)+t-\left\lfloor\frac{k}{2}\right\rfloor\right)+(-1)^k\sigma_k\left(-\sigma_{L}(w_j\cdot x+u_j)-t+\left\lfloor\frac{k}{2}\right\rfloor\right)\right]\\
				=&\sigma_{L}(w_j\cdot x+u_j),\quad j\leq n,\\
				h_{\ell+2,j}(x)=&0,\quad\hbox{otherwise.}
			\end{split}
		\end{equation*}
		We repeat these two steps up to the $L$-th layer, then the last hidden layer $h_L(x)$ is given as either
		\begin{equation*}
			\begin{split}
				h_{L,j}(x)=&\sigma_{K}(w_j\cdot x+u_j),\quad j\leq n,\\
				h_{L,j}(x)=&0,\quad\hbox{otherwise}
			\end{split}
		\end{equation*}
		or
		\begin{equation*}
			\begin{split}
				&h_{L,2(k+1)(m-1)+2t+1}(x)=\sigma_k\left(\sigma_{K}(w_m\cdot x+u_m)+t\right)\\
				&h_{L,2(k+1)(m-1)+2t+2}(x)=\sigma_k\left(-\sigma_{K}(w_m\cdot x+u_m)-t\right).
			\end{split}
		\end{equation*}
		for $m=1,\dots,n,\ t=0,\dots,k$.
		
		In either case, by \eqref{eqn:represent_id}, there exists $c\in\left[-\left(\frac{k}{2}+1\right)^4,\left(\frac{k}{2}+1\right)^4\right]^{2(k+1)m}$ such that
		$$c\cdot h_{L}(x)=\sum\limits_{m=1}^nc_m\sigma_{K}(w_m\cdot x+u_m)=f_n(x).$$
		
		By the form of the matrices $A_0,\dots,A_{L-1}$, the number of parameters is
		\begin{equation*}
			dn+2(k+1)n+\sum\limits_{j=2}^{L}\left[2(k+1)n+2(k+1)n\right]+2(k+1)n=N.
		\end{equation*}
		This completes the proof of Lemma \ref{lem:variation_space}.
		
	\end{proof}
	
	This lemma enables us to consider the rate of approximation for functions from the variation spaces generated by ReLU$^K$ activation functions with $K=k,k^2,\dots,k^L$. The approximation theory of such spaces has been studied in \cite{siegel2022sharp,ma2022uniform,klusowski2018approximation}, etc.
	\begin{definition}[Variation space]
		Let $K,d\in\NN^+$, denote the dictionary $\mathbb P_K^d$ as
		$$\mathbb P_K^d=\left\{x\mapsto\sigma_K(\omega\cdot x+b):\ \omega\in\mathbb S^{d-1},\ b\in[-1,1]\right\}.$$
		Consider the closure of the convex, symmetric hull of $\mathbb P_K^d$
		\begin{equation}
			\overline{\mathrm{Conv}(\mathbb P_K^d)}=\overline{\left\{\sum\limits_{j=1}^na_jh_j:\ n\in\NN^+,\ \sum\limits_{j=1}^n|a_j|\leq1,\ h_j\in\mathbb P_K^d\right\}},
		\end{equation}
		we define the variation space with the dictionary $\mathbb P_L^d$ as
		\begin{equation}
			\mathcal{K}_1(\mathbb P_K^d):=\left\{f\in R\times\overline{\mathrm{Conv}(\mathbb P_K^d)}:\ R>0\right\}.
		\end{equation}
		with the norm
		\begin{equation}
			\|f\|_{\mathcal{K}_1(\mathbb P_K^d)}=\inf\left\{R>0:\ f\in R\times\overline{\mathrm{Conv}(\mathbb P_K^d)}\right\}.
		\end{equation}
	\end{definition}

	We will compare the result with the rate of approximation of shallow ReLU$^k$ networks given in \cite{siegel2022sharp}.
	\begin{theorem}[Sigel and Xu, 2022]
		Let $k,d\in\NN^+$, then
		\begin{equation}
			\sup\limits_{f\in \mathcal{K}_1(\mathbb{P}_{k}^d)}\inf\limits_{f_n\in\Sigma_n^k(2)}\|f-f_n\|_{L^2(\BB^d)}\lesssim \|f\|_{\mathcal{K}_1(\mathbb{P}_{k}^d)}n^{-\frac{1}{2}-\frac{2k+1}{2d}}.
		\end{equation}
	\end{theorem}
	
	Combining this theorem with Lemma \ref{lem:variation_space}, we have the following theorem.
	\begin{theorem}\label{thm:variation}
		Let $k,n,L\in\NN^+$, and $n_1=\dots=n_{L}=2(k+1)n$, then there exists a deep {\rm ReLU}$^k$ network architecture $\Sigma_{n_{1:L}}^k(B)$ with $B=\left(\frac{k}{2}+1\right)^4$ such that for any $K=k^{\ell}$ with $1\leq\ell\leq L$,
		$$\sup\limits_{f\in \mathcal{K}_1(\mathbb{P}_{K}^d)}\inf\limits_{f_n\in\Sigma_{n_{1:L}}^k(B)}\|f-f_n\|_{L^2(\BB^d)}\leq C\|f\|_{\mathcal{K}_1(\mathbb P_K^d)}n^{-\frac{1}{2}-\frac{2K+1}{2d}},$$
		where $N=[(4L-2)(k+1)+d]n$, $C$ is a constant independent of $n$.
	\end{theorem}
	
	\begin{remark}
		The index $K$ in the variation spaces defined can be correlated to the regularity of the functions in the spaces (cf. \cite{siegel2022sharp}).
		Suppose the target function has a large but unknown regularity, saying $f\in \cap_{j=1}^K\mathcal{K}_1(\mathbb P_j^d)$, then the approximation property of shallow {\rm ReLU}$^k$ networks would be bad if the activation function is not chosen properly: If $k \ll K$, then the approximation rate that we can achieve is at most
		$$\sup\limits_{f\in \mathcal{K}_1(\mathbb P_k^d)}\inf\limits_{f_n\in\Sigma_n^k}\|f-f_n\|_{L^2(\BB^d)}\lesssim \|f\|_{\mathcal{K}_1(\mathbb P_k^d)}n^{-\frac{1}{2}-\frac{2k+1}{2d}}.$$
		If $k > K$, $\cap_{j=1}^K\mathcal{K}_1(\mathbb P_j^d)\not\subseteq\mathcal{K}_1(\mathbb P_k^d)$, the approximation rate is unknown.
		
		However, by taking the depth $L\geq\lceil\log_kK\rceil$, the deep {\rm ReLU}$^k$ networks we constructed in this section can automatically approximate functions from $\cap_{j=1}^K\mathcal{K}_1(\mathbb P_j^d)$ with this unknown $K$. The rate of approximation is at least
		$$\sup\limits_{f\in \mathcal{K}_1(\mathbb P_K^d)}\inf\limits_{f_n\in\Sigma_{n_{1:L}}^{k}}\|f-f_n\|_{L^2(\BB^d)}\leq C\|f\|_{\mathcal{K}_1(\mathbb P_K^d)}n^{-\frac{1}{2}-\frac{2K+1}{2d}},$$
		since $k^{\lceil\log_kK\rceil}\geq K$.
	\end{remark}

	\section{Conclusions}\label{sec:conclusions}
	In conclusion, this paper has significantly advanced the understanding of ${\rm ReLU}^k$
	neural networks and their capabilities in function approximation and representation. Our constructive proofs have not only demonstrated the representational power of shallow ${\rm ReLU}^k$ networks in polynomial approximation but also paved the way for constructing deep ${\rm ReLU}^k$
	networks with an explicit parameterization for any polynomial of degree less than $k^L$, scaling efficiently with network depth and dimensions. Furthermore, we established a connection between the coefficients of polynomials and the bounds of neural network parameters, thereby enabling precise estimations of network parameters that uphold the approximation integrity. The elucidation of deep ${\rm ReLU}^k$ networks' ability to approximate functions from Sobolev spaces, albeit with a suboptimal rate, opens new avenues for exploring their application in complex functional spaces. Most notably, our construction that deep ${\rm ReLU}^k$ networks can emulate the function representation of shallower networks with enhanced efficiency underscores the hierarchical strength of depth in neural architectures. Coupled with recent findings on variation spaces, our results underscore the adaptivity and potential of deep ${\rm ReLU}^k$ networks to achieve significant approximation accuracy, contributing to the theoretical foundations that may stimulate further breakthroughs in the application of deep learning across varied domains.
	
	\bibliographystyle{abbrv}
	\bibliography{main}
	
\end{document}